\documentclass{article}
\usepackage{PRIMEarxiv}
\usepackage{amsmath}

\usepackage[utf8]{inputenc} 
\usepackage[T1]{fontenc}    
\usepackage{hyperref}       
\usepackage{url}            
\usepackage{booktabs}       
\usepackage{amsfonts}       
\usepackage{nicefrac}       
\usepackage{microtype}      
\usepackage{lipsum}
\usepackage{fancyhdr}       
\usepackage{graphicx}       
\graphicspath{{media/}}     

\usepackage{natbib}

\usepackage[linesnumbered, ruled, vlined]{algorithm2e}
\usepackage{amsmath}
\usepackage{graphicx}
\usepackage{textcomp}
\usepackage{xcolor}
\usepackage{booktabs} 
\usepackage[utf8]{inputenc}
\usepackage{multirow}
\usepackage{algorithmic}
\usepackage{hyperref}
\usepackage{amsthm}
\usepackage{amsfonts}
\usepackage{xspace}
\usepackage{float}

\usepackage{subcaption}

\usepackage{bm}
\usepackage{mathtools}
\usepackage{stmaryrd}
\usepackage{stmaryrd}
\usepackage[disable]{todonotes}

\DeclareMathOperator*{\argmin}{\arg\!\min}

\newtheorem{theorem}{Theorem}

\newtheorem{observation}{Observation}
\newtheorem{problem}{Problem}
\newtheorem{definition}{Definition}

\newcommand{\squishlist}{
 \begin{list}{$\bullet$}
  {  \setlength{\itemsep}{0pt}
     \setlength{\parsep}{3pt}
     \setlength{\topsep}{3pt}
     \setlength{\partopsep}{0pt}
     \setlength{\leftmargin}{2em}
     \setlength{\labelwidth}{1.5em}
     \setlength{\labelsep}{0.5em}
} }
\newcommand{\squishlisttight}{
 \begin{list}{$\bullet$}
  { \setlength{\itemsep}{0pt}
    \setlength{\parsep}{0pt}
    \setlength{\topsep}{0pt}
    \setlength{\partopsep}{0pt}
    \setlength{\leftmargin}{2em}
    \setlength{\labelwidth}{1.5em}
    \setlength{\labelsep}{0.5em}
} }

\newcommand{\squishdesc}{
 \begin{list}{}
  {  \setlength{\itemsep}{0pt}
     \setlength{\parsep}{3pt}
     \setlength{\topsep}{3pt}
     \setlength{\partopsep}{0pt}
     \setlength{\leftmargin}{1em}
     \setlength{\labelwidth}{1.5em}
     \setlength{\labelsep}{0.5em}
} }

\newcommand{\squishend}{
  \end{list}
}

\newcommand{\note}[1]{}

\newcommand{\np}{{\sc NP}\xspace}
\newcommand{\bigO}{{\sc O}\xspace}

\SetKwComment{Comment}{// }{}

\newcommand{\ie}{i.e.\xspace}

\newcommand{\reals}{\ensuremath{\mathbb R}\xspace}

\newcommand{\eqcardpartition}{{\sc EqCardPartition}\xspace}

\newcommand{\svd}{{\sc SVD}\xspace}

\newcommand{\pca}{{\sc PCA}\xspace}
\newcommand{\fairpca}{{\sc FairPCA}\xspace}
\newcommand{\QR}{{\sc QR}\xspace}
\newcommand{\RRQR}{{\sc RRQR}\xspace}
\newcommand{\HRRQR}{{\sc H-RRQR}\xspace}
\newcommand{\LRRQR}{{\sc L-RRQR}\xspace}

\newcommand{\css}{{\sc CSS}\xspace}

\newcommand{\faircss}{{\sc FairCSS}\xspace}
\newcommand{\faircssminmax}{{\sc FairCSS-MinMax}\xspace}
\newcommand{\minfairnessscores}{{\sc MinFairnessScores}\xspace}
\newcommand{\kfairnessscores}{{\sc $c$-FairnessScores}\xspace}
\newcommand{\lscore}[2]{{\ensuremath \ell_{#2}^{(k)}(#1)}}
\newcommand{\minmaxloss}{{\sc MinMaxLoss}\xspace}

\newcommand{\slowqr}{{\sc S-Low QR}\xspace}
\newcommand{\shighqr}{{\sc S-High QR}\xspace}
\newcommand{\sgreedy}{{\sc S-Greedy}\xspace}
\newcommand{\lowqr}{{\sc Low QR}\xspace}
\newcommand{\highqr}{{\sc High QR}\xspace}
\newcommand{\greedy}{{\sc Greedy}\xspace}
\newcommand{\random}{{\sc Random}\xspace}
\newcommand{\sampler}{{\sc FairScoresSampler}\xspace}

\newcommand{\communities}{{\sf\small communities}\xspace}
\newcommand{\compas}{{\sf\small compas}\xspace}
\newcommand{\adult}{{\sf\small adult}\xspace}
\newcommand{\german}{{\sf\small german}\xspace}
\newcommand{\recidivism}{{\sf\small recidivism}\xspace}
\newcommand{\student}{{\sf\small student}\xspace}
\newcommand{\meps}{{\sf\small meps}\xspace}
\newcommand{\heart}{{\sf\small heart}\xspace}
\newcommand{\credit}{{\sf\small credit}\xspace}

\newcommand{\opt}{{\sf\small opt}\xspace}
\newcommand{\minmax}{{\sf\small minmax}\xspace}
\newcommand{\fair}{{\sf\small fair}\xspace}

\title{Fair Column Subset Selection
}

\author{
  Antonis Matakos \\
  Aalto University \\
  Espoo, Finland \\
  \texttt{antonis.matakos@aalto.fi} \\
   \And
  Bruno Ordozgoiti \\
  Unaffiliated \\
  London, United Kingdom\\
  \texttt{bruno.ordozgoiti@gmail.com} \\
  \And
  Suhas Thejaswi \\
  Max Planck Institute for Software Systems \\
  Kaiserlslautern, Germany\\
  \texttt{thejaswi@mpi-sws.org}
}

\begin{document}
\maketitle

\begin{abstract}
The problem of column subset selection asks for a subset of columns from an input matrix such that the matrix can be reconstructed as accurately as possible within the span of the selected columns. A natural extension is to consider a setting where the matrix rows are partitioned into two groups, and the goal is to choose a subset of columns that minimizes the maximum reconstruction error of both groups, relative to their respective best rank-$k$ approximation. Extending the known results of column subset selection to this fair setting is not straightforward: in certain scenarios it is unavoidable to choose columns separately for each group, resulting in double the expected column count. We propose a deterministic leverage-score sampling strategy for the fair setting and show that sampling a column subset of minimum size becomes \np-hard in the presence of two groups. Despite these negative results, we give an approximation algorithm that guarantees a solution within 1.5 times the optimal solution size. We also present practical heuristic algorithms based on rank-revealing QR factorization. Finally, we validate our methods through an extensive set of experiments using real-world data.
\end{abstract}

\section{Introduction} \label{sec: Intro} 
Dimensionality reduction techniques such as principal component analysis (PCA) and non-negative matrix factorization have proven useful for machine learning and data analysis tasks~\citep{pearson1901lines,lee1999learning}. Such tasks include feature selection, feature extraction, noise removal and data visualization, among others. These techniques are also commonly employed as components of larger machine learning (ML) pipelines to simplify data through lower-dimensional representations. However, when the data is divided into subsets (or groups), an inaccurate low-dimensional representation of any subset can perpetuate inaccuracies in subsequent downstream tasks. 
Therefore, there is increasing emphasis on developing techniques that produce accurate representations for all the different groups.

Notably, \citet{samadi2018price} showed that a well-known dimensionality reduction technique, PCA, may incur higher average reconstruction error for a subset of the population than the rest, even when the groups are of similar size. They proposed fair variants of PCA, where the objective is to minimize the maximum reconstruction error for any group. However, a drawback of PCA is that the results are often hard to interpret, since its output consists of abstract attributes that might not necessarily be part of the input data. As an alternative to PCA, one may ask for algorithms that choose a (small) subset of the original attributes of the dataset to act as a low dimensional representation.

In the {\em column subset selection problem} (\css), we are given a data matrix and seek a representative subset of its columns. The quality of the solution is measured by the residual norm when the input matrix is projected onto the subspace spanned by the chosen columns. \css has been extensively studied, and polynomial-time approximation algorithms are known for different quality criteria~\citep{deshpande2006matrix, deshpande2010efficient, boutsidis2009improved, boutsidis2014near, papailiopoulos2014provable, altschuler2016greedy}.

In this work, we study the {\em fair column subset selection problem} (\faircss), where the rows of the data matrix are partitioned into two groups. The goal is to choose a subset of columns that minimize the maximum reconstruction error of both groups, through a min-max objective. This distinguishes our problem from conventional \css methods that aim to minimize the reconstruction error of the entire data matrix and may neglect either group.
Here, we focus on the two-group setting. Generalizing our results to an arbitrary number of groups is a non-trivial technical leap and thus left as future work. Nevertheless, binary protected attributes are encountered often in practice and are commonly the focus of works on algorithmic fairness~\citep{samadi2018price, chierichetti2017fair}.

{\em To further motivate the study of \faircss, consider its application in drug discovery.} Suppose we are given a large dataset of medical records, where each row represents a patient and the columns contain genetic information indicating presence (absence) of a gene expression with binary value $1$ ($0$). Due to the size of the genetic data, it is reasonable to ask for a succint representation of the dataset. Using \pca would yield an arbitrary subspace of the data that is hard to interpret. On the other hand, classical \css would return a representative set of columns (genes), however, if there is bias in the medical records then the reconstruction error could skew favouring a (majority) group. Thus, any subsequent insights derived from the biased (succint) data could also perpetuate these biases. It is possible that, a drug discovered using this biased data is more (less) effective for men (women), based on the data majority. Our proposed method, however, returns a subset of columns that are representative of both groups, thus mitigating potential biases and ensuring fair representation in subsequent analysis.

The two-group setting of \css introduces significant challenges, making it necessary to view the problem through a new lens.
In \css, finding the optimal $k$-column subset is \np-hard~\citep{shitov2021column} and thus, polynomial-time approximation algorithms are usually sought. A factor of $O(k)$ with respect to the best rank-$k$ approximation is possible if we are allowed to choose at most $k$ columns~\citep{deshpande2010efficient}, and better results can be obtained if we allow $c \geq k$ columns~\citep{boutsidis2008unsupervised}. In the two-group fair setting, similar approximation bounds can be reproduced if we allow twice as many columns in the solution; by optimizing for each group separately.
However, addressing the two groups separately can raise ethical and legal concerns~\citep{lipton2018mitigating,samadi2018price}, potentially conflicting with the principles of demographic parity and equal treatment~\citep{barocas2019fairness}. Further, this approach may inadvertently contribute to segregation and stereotyping. Therefore, our focus is to find a {\em common} subset of representative columns for {\em both} groups.

Unfortunately, as we show with an example (see \S~\ref{sec:problem}), in certain scenarios it may not be possible to do better than choosing twice as many columns. Can we hope to obtain guarantees of any kind in this fair setting? We answer this question affirmatively. To achieve this, we adapt {\em leverage score sampling} to our fair setting. Leverage scores are obtained using the SVD of the matrix, and they are used to find provably good column subsets for \css. In particular, \citet{papailiopoulos2014provable} showed that by choosing columns with highest leverage scores such that their sum adds to a predefined threshold $\theta$, a relative-error approximation is possible.
In standard \css, finding a column subset of minimum size satisfying threshold $\theta$ is trivial, by sorting the leverage scores. In the fair setting however, finding a subset of leverage scores of minimum size satisfying threshold $\theta$ for both groups is \np-hard. While the original result can be extended to the fair setting by solving for both groups independently and doubling the column count, we present an efficient algorithm that achieves this with essentially $1.5$ as many columns. Whether this factor can be improved is left as an open question. Finally, we introduce efficient heuristics for the problem based on QR factorizations with pivoting, and assess their performance in our empirical evaluation.
Our contributions are summarized as follows:
\squishlisttight
\item We introduce the novel problem of fair column subset selection, where two groups are present in the data and the maximum reconstruction error of the two groups must be minimized. 
\item We extend the approach of deterministic leverage score sampling to the two-group fair setting. We show that the smallest column subset that achieves the desired guarantees is \np-hard to find, and give a polynomial-time algorithm wih relative-error guarantees with a column subset of essentially $1.5$ times the minimum possible size.
\item We present efficient heuristics based on QR factorizations with column pivoting.
\item We empirically evaluate our algorithms on real-world data and show they are able to select fair columns with high accuracy. Further, we analyse the price of fairness for our formulation.
\squishend
The paper is organised as follows: Section~\ref{sec:related} covers related work. Section~\ref{sec:problem} outlines our problem. We adapt deterministic leverage score sampling to the fair setting in Section~\ref{sec:pairs_ls}. Algorithms based on QR decomposition are given in Section~\ref{sec:fair_qr}. Experimental results are presented in Section~\ref{sec:experiments}. Finally, we conclude in Section~\ref{sec:conclusion}.

\section{Related Work} \label{sec:related}
Our work builds upon related work in the areas of algorithmic fairness and column subset selection.

{\bf Algorithmic fairness.} 
The influential work of \citet{dwork2012fairness} established a formal notion of fairness in algorithmic decision-making, which has served as a foundation for subsequent research in the field. \citet{pedreshi2008discrimination} addressed discrimination in data mining, while \cite{kamiran2010classification, kamiran2012data} proposed a framework for mitigating discrimination in classification. Since then, there has been extensive research on algorithmic fairness from various disciplines, including economics, game theory, statistics, ethics, and computer science~\citep{rambachan2020economic, saetra2022ai, raji2022actionable, buolamwini2018gender}. Fair objectives have been introduced into many classical computer science problems~\citep{chierichetti2017fair, samadi2018price, anagnostopoulos2020spectral, ghadiri2021socially, froese2022modification, zemel2013learning, zafar2017fairness, zehlike2017fair}.
Fair clustering share similarities to fair CSS, with a key distinction: the former asks for a subset of fair representative data points (rows), while the later focuses on finding a fair representative column subset (attributes)~\cite{chierichetti2017fair,chhabra2021overview,ghadiri2021socially}.
For further reading see \cite{pessach2022review, mitchell2021algorithmic, kleinberg2018algorithmic}.

{\bf Column Subset Selection.}
\css can be viewed as a method for feature selection. Early work in \css traces back to the numerical linear algebra community, and the seminal works of \citet{golub1965numerical} on pivoted QR factorizations, which was followed by works by \citet{chan1987rank,hong1992rank,chan1994low} addressing the problem of finding efficient {\em rank revealing QR factorizations} (RRQR). Recently, \css has attracted interest in the computer science community, with approaches that combine sampling with RRQR~\citep{boutsidis2009improved, papailiopoulos2014provable}, while greedy methods have proved highly effective~\citep{altschuler2016greedy,farahat2015greedy}.

{\bf Fair PCA.} 
Closest to our work is fair principal component analysis (\fairpca) by \citet{samadi2018price}, where the goal is to find a low-dimensional projection of the data to optimize a min-max objective that ensures fairness with respect to two groups. \citet{tantipongpipat2019multi} extended \fairpca to a multi-objective setting. \citet{olfat2019convex} presented a polynomial-time algorithm using convex programming for the general case with multiple groups. Similar to \fairpca, we employ a min-max objective, but in contrast, we want to find a subset of actual columns to approximate the reconstruction error of both groups. In fact, we show that the problem we define, fair \css, is \np-hard for any number of groups (see \S~\ref{sec:problem}).

\section{Problem statement} \label{sec:problem}
In this section, we describe relevant terminology and formally define \css before introducing fair \css. For a positive integer $n$ we denote $[n] = \{1,\ldots,n\}$. For any matrix $C$, $P_C$ denotes the projection operator onto the subspace spanned by the columns of $C$.
\begin{problem}[Column Subset Selection (\css)]
\label{prob:CSS}
Given a matrix $M \in \mathbb{R}^{m\times n}$ and a positive integer $k$. The goal is to choose $k$ columns of $M$ to form a matrix $C \in \reals^{m \times k}$ such that the reconstruction error
$$loss(M,C) = \|M - P_C M\|_{F}$$
is minimized, where $\|\cdot\|_F$ denotes the Frobenius norm.
\end{problem}
{\bf Fairness.}
A solution to \css gives the best possible column approximation for the matrix $M$ {\em overall}. However, if the matrix rows are partitioned into two groups, it is possible that the reconstruction error shows significant disparity when measured on each group separately, for instance, if one of the groups is a minority. In the fair setting, our goal is to choose a subset of columns that achieves good reconstruction error for both groups. To formalize this, assume that the rows of matrix $M$ are partitioned into two subsets $A$ and $B$. Subscripting a matrix by $A$ or $B$ denotes the rows corresponding to $A$ or $B$. For ease of notation, we override $M_A=A$ and $M_B=B$.

Before presenting our fair \css formulation, we make two considerations. In \css, the optimal projection is obtained as $P_C=CC^+$, where $C^+$ is the pseudoinverse of $C$. In our setting, even though we choose a common column subset for both groups, the optimal reconstruction error for each group is attained by a different projection operator, that is, $C_AC_A^+A$ and $C_BC_B^+B$.
Second, we are interested in minimising the \emph{relative loss} of each group, with respect to the optimal reconstruction error, which can be obtained using the best rank-$k$ approximation. This is to avoid excessively penalizing either of the two groups, when the other group does not have a good low-rank representation. Note that measuring the reconstruction error relatively to the best rank-$k$ approximation is common in \css literature. See \citet{boutsidis2008unsupervised, papailiopoulos2014provable}.

\begin{definition}[Relative group-wise reconstruction error]
\label{def:rec_loss}
Given a matrix $M \in \reals^{m \times n}$, a row subset $A \in \reals^{a \times n}$, with $A_k$ its best rank-$k$ approximation, and a matrix $C \in \reals^{m \times k}$ formed by choosing $k$ columns of $M$. The relative reconstruction error of $A$ is:
\[ 
Nloss_A(M,C) = \frac{\| A - C_A C_A^+ A\|_{F}}{\| A - A_k \|_{F}}.
\]
\end{definition}
Now we formally define our problem.
\begin{problem}[\faircssminmax]
\label{prob:faircss}
Given a matrix $M \in \mathbb{R}^{m\times n}$ with row partition $A, B$, and a positive integer $k$. The goal is to choose $k$ columns of $M$ to form a matrix  $C \in \reals^{m \times k}$ 
that optimizes the objective:
\[
\min_{\substack{C \in \mathbb{R}^{m \times k}\\ C \subset M}} \max \biggl\{ 
\frac{\| A - C_A C_A^+ A\|_F}{\| A - A_k \|_F},
\frac{\| B - C_B C_B^+ B\|_F}{\| B - B_k \|_F} \biggl\}.
\]
\end{problem} 
Thus, we search for a subset of columns that minimizes the maximum error of either groups.
Observe that when groups $A$ and $B$ are identical, \faircssminmax is equivalent to \css, which is known to be \np-hard~\citep{shitov2021column}[Theoreom~2.2]. The hardness results extend to \faircssminmax. In line with preceding arguments, \faircssminmax is \np-hard for any number of groups.
\begin{observation}
\faircssminmax is \np-hard.
\end{observation}
{\bf Limitations.}
We show that, in some scenarios it is not possible to do better than solving for two groups separately. In particular, any algorithm that attempts to solve \css for more than one group, \ie, to achieve errors smaller than $\|A\|_F^2$ and $\|B\|_F^2$ for $A$ and $B$, respectively, cannot achieve meaningful bounds on the relative-error with respect to both $\|(A)_k\|_F^2$ and $\|(B)_k\|_F^2$. Consider the following example:
\[
\left (
\begin{matrix}
X_A & 0_A \\
0_B & X_B
\end{matrix}
\right ),
\]
where $X_A,X_B$ are matrices of rank $n>k$.
If we pick less than $2k$ columns, the error is at least
$\min\left\{\frac{\sum_{i=k}^n\sigma_i^2(A)}{\sum_{i=k+1}^n\sigma_{i}^2(A)}, 
\frac{\sum_{i=k}^n\sigma_{i}^2(B)}{\sum_{i=k+1}^n\sigma_{i}^2(B)}\right\}$, 
where $\{\sigma_1, \dots, \sigma_n\}$ are the singular values with $\sigma_i \geq \sigma_{i+1}$. Thus, the relative error can be unbounded if the rank of either submatrix is numerically close to $k$, \ie, if $\sigma_k(A) \gg \sigma_{k+1}(A)$ or $\sigma_k(B) \gg \sigma_{k+1}(B)$. Clearly, the only way to prevent this is to pick $2k$ columns.
A similar result is yielded by matrices where the blocks of zeroes are replaced by small values. Despite this, in the following section we present an algorithm with a bounded-error relative to the best rank-$k$ approximation, by relaxing the requirement on the number of selected columns.

\section{Pairs of leverage scores} \label{sec:pairs_ls}
In this section, we discuss leverage scores in the context of the two-group fair setting. Leverage scores are a pivotal concept that is extensively studied in \css literature, as they provide valuable insights for selecting column subsets with provable approximation guarantees. We begin by discussing the results of \cite{papailiopoulos2014provable} on leverage score sampling. Next, we show that sampling a subset of columns of minimum size in the two-group fair setting is \np-hard. Finally, we present an approximation algorithm that samples $1.5$ times the size of an optimal solution.
An useful concept that is extensively studied in \css literature is \emph{leverage scores}, which is defined as follows:
\begin{definition}[Leverage scores]
Let $V^{(k)} \in \reals^{n \times k}$ denote the top-$k$ right singular vectors of a $m \times n$ matrix $M$ with rank $\rho = rank(M)\geq k$. Then, the rank-$k$ leverage score of the $i$-th column of $M$ is defined as:
\[
\ell_i^{(k)}(M)=\|[V^{(k)}]_{i,:}\|_2^2 \text{~~for all~~} i \in [n].
\]
Here, $[V^{(k)}]_{i,:}$ denotes the $i$-th row of $V^{(k)}$. 
\end{definition}
Leverage scores are used to find a solution with approximation guarantees for \css. 
In particular, we focus on the following result by \citet{papailiopoulos2014provable}.
\begin{theorem}[\cite{papailiopoulos2014provable}]
\label{theorem: approximation}
Given a matrix $M \in \reals^{m \times n}$ and an integer $k<\text{rank}(M)$. Let $\theta=k-\epsilon$ for some $\epsilon \in (0,1)$ and $S$ be a subset of column indices such that $\sum_{i \in S} \ell_i^{(k)} \geq \theta$, and $C\in \reals^{m \times k}$ be the matrix of $M$ formed by choosing the columns with indices in $S$.
Then we have that
\[
\| M - C C^+M \|_{F}^2 \leq (1-\epsilon)^{-1} \| M - M_k \|_{F}^2.
\]
\end{theorem}
In essence, the above result implies that by selecting a column subset whose leverage scores sum to at least threshold $\theta$, we obtain a relative error guarantee with respect to the best rank-$k$ approximation. They proposed a deterministic algorithm that picks $c \geq k$ columns with the largest leverage scores that sum to at least threshold $\theta$. The algorithm runs in time~$\bigO(\min\{m,n\}mn)$. 

Depending on the leverage score distribution, sometimes it may be necessary to pick more than $k$ columns to satisfy the threshold $\theta$, and $\Omega(n)$ in the worst case.
Nevertheless, when the leverage scores follow power-law decay, a small factor of $k$ suffices~\citep{papailiopoulos2014provable}[Theorem~3].

{\bf Fair deterministic leverage-score sampling.}
In order to achieve approximation guarantees for both groups, the leverage scores of both $A$ and $B$ must sum to at least $\theta$, individually. While seeking the minimum number of columns to satisfy the threshold is trivial in the single-group setting, it becomes \np-hard in the presence of two groups. Let us formally define the problem and analyse its complexity.
\begin{problem}[\sc{Min-FairnessScores}]
\label{prob:min-fairnessscores}
Given matrices $A \in \mathbb{R}^{m_A\times n}$ and $B  \in \mathbb{R}^{m_B\times n}$, $k \in \mathbb{N}: 0<k<\text{rank}(A),\text{rank}(B)$ and a threshold $\theta =k-\epsilon$ for some $\epsilon \in (0,1)$, find the smallest set of indices $S \subseteq [n]$ such~that:
\[
\sum_{i\in S} \ell_i^{(k)}(A) \geq \theta \text{~and~}
\sum_{i \in S} \ell_i^{(k)}(B) \geq \theta. 
\]
\end{problem}
If we find a subset of columns that satisfy both inequalities above, then due to Theorem~\ref{theorem: approximation}, we have that,
\begin{align*}
\|A-C_AC_A^+A\|_F \leq (1-\epsilon)^{-1/2} \|A-A_k\|_F \implies
Nloss_A(M,C) \leq (1-\epsilon)^{-1/2}.
\end{align*}
and similarly, $Nloss_B(M,C) \leq (1-\epsilon)^{-1/2}.$

A solution to Problem~\ref{prob:min-fairnessscores} gives us an upper bound on the reconstruction error. Unfortunately, \minfairnessscores is \np-hard. 
\begin{theorem} \label{theorem:nphard}
\minfairnessscores is \np-hard.
\end{theorem}
\begin{proof}
To establish hardness we reduce the equal cardinality partition problem~(\eqcardpartition) to a decision version of \minfairnessscores, called \kfairnessscores. The decision version asks to find \emph{exactly} $c$ indices i.e., $S \subseteq [n]$, $|S|=c$ such that $\sum_{i \in S} \ell_i^{(k)}(A) \geq \theta$ and $\sum_{i\in S} \ell_i^{(k)}(B) \geq \theta$. 
Given a set $Z=\{p_1,\ldots,p_n\}$ of $n$ positive integers, \eqcardpartition asks to partition $Z$ into two disjoint subsets $X, Y$ such that $X \cup Y = Z$, $|X|=|Y|$ and $\sum_{p_i \in X} p_i= \sum_{p_j \in Y} p_j$. \eqcardpartition is known to be \np-complete~\cite[SP12]{garey1979computers}.

Given an instance of \eqcardpartition $(Z,n)$, we reduce it to a \kfairnessscores instance $(A, B, \theta, c)$ as follows. 
Let $s=\sum_{p_i \in Z} p_i$ and $M \gg s$ be some constant. 
Let $A=[\sqrt{p_1},\ldots,\sqrt{p_n}]$ and $B=\big[\sqrt{M-\frac{p_1}{s}},\ldots,\sqrt{M-\frac{p_n}{s}} \big]$ be input matrices such that $A \in \mathbb{R}^{n\times 1}$, $B \in \mathbb{R}^{n\times 1}$. Finally, we set $\theta= 1/2$ and $c=n/2$. We claim that \eqcardpartition is a {\sc yes} instance if and only if \kfairnessscores is a {\sc yes} instance. The reduction is polynomial in the input size.

To make the reduction work, we need a way to map the values in the instance \eqcardpartition to leverage scores, which are obtained from \svd. The most straightforward way is to compute the rank-$1$ leverage scores. For all $i \in [n]$ we can get,
\[
\big( \ell_i^{(1)}(A),  \ell_i^{(1)}(B)\big) = 
\Big( \frac{p_i}{s},\frac{M-\frac{p_i}{s}}{nM-1} \Big).
\]
For ease of notation, let $\alpha_i = \ell_i^{(k)}(A)$ and
$\beta_i = \ell_i^{(k)}(B)$.

Let \kfairnessscores be a {\sc yes} instance. Then we have a subset $S \subseteq [n]$ such that $|S| = n/2$,
\begin{equation*}
\sum_{i \in S} \alpha_i  = \frac{\sum_{i \in S} p_i}{s} \geq \frac{1}{2} 
\text{~which implies } \sum_{i \in S} p_i \geq \frac{s}{2},
\end{equation*}
\begin{equation*}
\sum_{i \in S} \beta_i = \frac{\sum_{i \in S} (M-\frac{p_i}{s})}{nM-1} \geq
\frac{1}{2}
\text{~which implies } \sum_{i \in S} p_i \leq \frac{s}{2}.
\end{equation*}
For both of the above inequalities to hold simultaneously, we must have $\sum_{i \in S} p_i = s/2$. Thus, $X=\{p_i: i \in S\}, Y=\{p_i: i \in [n] \setminus S\}$ is a solution to
\eqcardpartition.

For the sake of contradiction, assume that \kfairnessscores is a {\sc no} instance and $X, Y$ are a solution to \eqcardpartition. Let $S$ be the set of indices of elements in $X$. We can choose $S$ as solution for \kfairnessscores since $|S|=n/2$, 
$$\sum_{i \in S} \alpha_i = \sum_{i \in S} p_i/s = \frac{1}{2}, \text{ and} \sum_{i \in S} \beta_i = \frac{\sum_{i \in S}(M-\frac{p_i}{s})}{nM-1} = \frac{1}{2},$$
which is a contradiction. Thus, \eqcardpartition must be a {\sc no} instance.

Given a solution $S \subseteq [n]$ for \kfairnessscores, we can verify in polynomial time if $|S|=c$, compute \svd as well as check if $\sum_{i \in S} \alpha_i \geq \theta$ and $\sum_{i \in S} \beta_i \geq \theta$. Thus, \kfairnessscores is \np-complete.
Naturally, \kfairnessscores reduces to \minfairnessscores, as the latter finds the smallest $c$ such that a solution to \kfairnessscores exists. 
Thus, \minfairnessscores is \np-hard.
\end{proof}
Even though \minfairnessscores is \np-hard, a $2$-factor approximation is trivial: sort $\lscore{A}{i}$'s in decreasing order and choose indices with highest leverage scores until they sum to $\theta$. Repeat the same for $\lscore{B}{i}$'s. This results in at most $2c$ columns, where $c$ is the optimal number of columns to satisfy $\theta$.
\begin{algorithm}[h]
\small
\caption{\label{alg:fairnessScores}\sampler}
\DontPrintSemicolon
\KwIn{$P=\{(\alpha_1,\beta_1),\ldots,(\alpha_n,\beta_n)\}$, $\theta$}
\KwOut{$S \subseteq P$}

$S \gets \emptyset$, $Q \gets \emptyset$\\

\Comment*[l]{add $(\alpha_j, \beta_j)$ to $S$ until either of the thresholds is satisfied}
\While{$\sum_{(\alpha_i,\beta_i) \in S} \alpha_i < \theta$ and $\sum_{(\alpha_i,\beta_i) \in S} \beta_i < \theta$} {
	$(\alpha_j, \beta_j) \gets \max_{(\alpha_j,\beta_j) \in P \setminus S} (\alpha_j + \beta_j)$\\
	$S \gets S \cup (\alpha_j,\beta_j)$ 
}

\Comment*[l] {find $Q \in P\setminus S$ such that the other threshold is satisfied}
\If{$\sum_{(\alpha_i, \beta_i) \in S} \alpha_i \geq \theta$} {
  $Q \gets \argmin_{|Q|} \left(\sum_{j \in Q} \beta_j \geq \theta \right)$
} \Else {
$Q \gets \argmin_{|Q|} \left(\sum_{j \in Q} \alpha_j \geq \theta \right)$
}

$S \gets S \cup Q$\\
\Return{$S$}
\end{algorithm}

Next, we present an algorithm for \minfairnessscores that returns at most $\lceil 3c/2 \rceil + 1$ columns ($\approx 1.5$-approximation).\footnote{An additional column ($+1$) is required in case $c$ is odd.} The pseudocode is in Algorithm~\ref{alg:fairnessScores}.
For ease of notation, we denote $\alpha_i = \lscore{A}{i}$ and $\beta_i = \lscore{B}{i}$. The algorithm proceeds in two stages. In the first stage, at each iteration add to $S$ the index $i$ such that the cumulative gain $\alpha_i + \beta_i$ is maximized until a step $t \leq n$, where at least one of the inequalities is satisfied, 
i.e., either
$\sum_{i \in S} \alpha_i \geq \theta$ or
$\sum_{i \in S} \beta_i \geq \theta$.
In the second step, sort the tuples of leverage scores in $[n] \setminus S$ based on their contribution to the unsatisfied inequality, in descending order. Finally, pick the rest of the tuples based on this order, until the threshold $\theta$ is satisfied.
The following theorem establishes our approximation result. Note that the approximation is in terms of the number of columns $c$ in the optimal solution, and we have already established an $(1-\epsilon)^{-1/2}$ approximation in terms of our objective function.
\begin{theorem}
\label{theorem:approx}
Algorithm~\ref{alg:fairnessScores} returns a solution of at most $\lceil\frac{3}{2}c\rceil+1$ columns for \minfairnessscores, where $c$ is the number of columns in the optimal solution.
\end{theorem}
\begin{proof}
Given $P=\{(\alpha_1, \beta_1),\ldots,(\alpha_n,\beta_n)\}$, the task is to find the smallest subset $S \subseteq P$ such that
$\sum_{i \in S} \alpha_i \geq \theta$ and
$\sum_{i \in S} \beta_i \geq \theta$.
At each iteration we select a tuple with the maximum contribution, \ie, $\alpha_j+\beta_j$ until some step $t$, where either $\sum_{i \in S} \alpha_i \geq \theta$ or $\sum_{i \in S} \beta_i \geq \theta$ is satisfied.
Without loss of generality, at step $t$ we assume that
$\sum_{(\alpha_i, \beta_i) \in S} \alpha_i \geq \theta$.
Let $S^*$ be the optimal solution and $(\alpha_i^*,\beta_i^*)$ denote the contribution of $i$-th tuple in $S^*$. We assume $S^*$ is sorted in decreasing order according to $\alpha_i^*+\beta_i^*$. We note that this assumption makes our analysis easier without losing generality.

First we establish that $t \leq c$. Assume for contradiction that $t > c$ and that $\sum_{i=1}^t \alpha_i < \theta$ and that $\sum_{i=1}^t \beta_i < \theta$. Thus the algorithm has not terminated. 
From the optimality of $\alpha_i + \beta_i$ at step $t$ it holds that,
\begin{align*}
\sum_{i=1}^t \alpha_i & + \sum_{i=1}^t \beta_i \geq \sum_{i=1}^t \alpha_i^* +\sum_{i=1}^t \beta_i^*,\\
\sum_{i=1}^t \alpha_i  & \geq  \sum_{i=1}^c \alpha_i^* + \sum_{i=1}^c \beta_i^* - \sum_{i=1}^t \beta_i \geq 2\theta -\sum_{i=1}^c \beta_i \geq \theta,
\end{align*}
which is a contradiction, since we assumed $\sum_{i=1}^t \alpha_i < \theta$. Thus if $\sum_{i=1}^t \alpha_i \geq \theta$ then $1 \leq t \leq c$. We now discern two cases for the value of $t$.

{\bf Case  $t\leq \lceil\frac{c}{2}\rceil+1$:} 
We have $\sum_{i=1}^t \alpha_i \geq \theta$. To satisfy the second
inequality we choose tuples in $P \setminus S$ in decreasing order according to
$\beta_i$, which is at most $c$ columns, since the optimal
solution has $c$ columns. So, we have a solution with size
$|S| \leq t+c=\lceil\frac{3}{2}c\rceil+1$.

{\bf Case $\lceil\frac{c}{2}\rceil+1 < t \leq c$:}
Again, from the optimality of $\alpha_i + \beta_i$ at each step in $1 \leq t
\leq c$ we have,
\begin{align*}
\sum_{i=1}^{t-1} \alpha_i  + \sum_{i=1}^{t-1} \beta_i &\geq 
  \sum_{i=1}^{t-1} \alpha_i^* + \sum_{i=1}^{t-1}\beta_i^* \\
	& \geq \sum_{i=1}^{\lceil\frac{c}{2}\rceil} \alpha_i^* + 
  \sum_{i=\lceil\frac{c}{2}\rceil+1}^{t-1} \alpha_i^* + 
  \sum_{i=1}^{\lceil\frac{c}{2}\rceil} \beta_i^* +
  \sum_{i=\lceil\frac{c}{2}\rceil+1}^{t-1} \beta_i^* \\
	& \geq \theta +\sum_{i=\lceil\frac{c}{2}\rceil+1}^{t-1} \alpha_i^* +
  \sum_{i=\lceil\frac{c}{2}\rceil+1}^{t-1} \beta_i^*
\end{align*}
The third step follows from the assumption that the optimal solution $S^*$ has
tuples sorted in decreasing order according to $\alpha_i^*+\beta_i^*$. We also observe that
$\sum_{i=1}^{t-1} \alpha_i = \theta - \alpha_t $. Therefore we have
\[
	\sum_{i=1}^{t-1} \beta_i \geq \sum_{i=\frac{c}{2}+1}^{t-1} \alpha_i^* +
                            \sum_{i=\frac{c}{2}+1}^{t-1} \beta_i^* + \alpha_t
\]
Suppose $t=\lceil\frac{c}{2}\rceil+1+q$. This means we can afford to
add at most $c-q$ columns to our solution if we want to satisfy our bound on
the cardinality of $S$.

At this stage, the algorithm picks the columns with the largest values of $\beta_i^*$. This means that from those in the optimal solution, we miss at most $q$ $\beta_i^*$'s from among the bottom ones.
From above, we have
\[
	\sum_{i=1}^{t-1} \beta_i \geq \sum_{i=\frac{c}{2}+1}^{t-1} \alpha_i^* +
        \sum_{i=\frac{c}{2}+1}^{t-1} \beta_i^* \geq
         \sum_{i=c-q}^{c} \alpha_i^* + \sum_{i=c-q}^{c} \beta_i^* \geq \sum_{i=c-q}^{c} \beta_i^*.
\]
This holds since tuples of the optimal solution are sorted in decreasing order of the value of the pair sums, which implies that the value we miss from not adding the last $q$ $\beta_i^*$'s is already covered by
what we had, $\sum_{i=1}^{t-1} \beta_i$, so
$ \displaystyle
\sum_{i=1}^{|S|} \beta_i \geq \sum_{i=1}^{t-1} \beta_i + \sum_{i=1}^{c-q} \beta_i^*
\geq \theta.
$
So the solution has at most 
$t-1+c-q = \lceil\frac{c}{2}\rceil + q + c - q = \lceil\frac{3}{2}\rceil c + 1$
columns and satisfies the threshold.
\end{proof}
Even though Algorithm~\ref{alg:fairnessScores} offers an upper bound on the number of columns in the optimal solution $c$, depending on the task at hand, it may be undesirable to obtain more than $k$ columns. \css with exactly $k$ columns (the definition of Problem \ref{prob:CSS}) is a well-studied problem and a wide range of algorithms have been developed for it. These algorithms are typically based on \QR-decomposition. Motivated by this, in the following section we propose two \QR-decomposition-related algorithms for \faircss. 

Recall the impossibility results from Section~\ref{sec:problem}: unless we pick $2k$ columns it may be impossible to achieve a relative-error approximation in terms of the rank-$k$ reconstruction error. Thus, the following algorithms are heuristics. They can be used either directly for {\faircssminmax} or part of a two-stage approach, that we describe in detail in Section~\ref{sec:experiments}.

\section{Fair \QR decompositions} \label{sec:fair_qr}
Numerous practical algorithms for \css originate from numerical linear algebra, often relying on \QR decomposition with column pivoting. 
\begin{definition}[QR decomposition with column pivoting]
Given a matrix $M \in \mathbb{R}^{m \times n}$ with $m \geq n$ and an integer $k \leq n$. Matrix $M$ can be expressed as the product of an orthonormal matrix $Q \in \mathbb{R}^{m \times m}$ and an upper triangular matrix $R \in \mathbb{R}^{n \times n}$. More precisely,
\[
M\Pi=QR=Q \begin{pmatrix}
R_{11} & R_{12} \\ 0 & R_{22}
 \end{pmatrix},
\]
where $R_{11} \in \mathbb{R}^{k \times k}$, $R_{12} \in \mathbb{R}^{k \times (n-k)}$, $R_{22} \in \mathbb{R}^{(n-k)\times(n-k)}$ and $\Pi \in \mathbb{R}^{n \times n}$ is a permutation matrix. 
\end{definition}
Column pivoting involves finding a permutation matrix $\Pi$ for a given matrix $M$, such that $\|R_{22}\|_F$ is minimised. When this comes with certain guarantees, it leads to a \emph{rank revealing QR decomposition} ({\sc RRQR}), which forms the basis for various algorithms with approximation guarantees in \css. Specifically, if we denote the first $k$ columns of the permutation matrix $\Pi$ as $\Pi_k$, then choosing the column subset $C=M\Pi_k$ guarantees that $loss(M,C)$ is bounded \ie, $\|M-P_CM\|_F=\|R_{22}\|_F$~\citep{boutsidis2009improved}. We tailor Low-\RRQR (\LRRQR), originally introduced by \cite{chan1994low}, to accommodate the two-group fair setting.

We present a brief description of \LRRQR. We start with the \QR-decomposition of matrix $M$ to obtain matrices $Q$ and $R$. We initialise $R_{11}=0, R_{22}=R$ and build $R_{11}$ incrementally through column pivoting: at each iteration we permute a column of $R_{22}$ to the first position, through $\Pi$. Then, we compute the \QR-decomposition again,  drop the first row and column of the resulting $R$, and proceed recursively on it. The SVD serves in finding the pivot column: if $v \in \mathbb{R}^n$ is the right singular vector corresponding to the largest singular value, then successive permutations such that $|(\Pi^Tv)|_1=\|v\|_{\infty}$ lead to a provably small $\|R_{22}\|_F$.

{\bf Fair pivoting.}
Note that \LRRQR may introduce unfairness, since we factorize the matrix $M$, without considering the error of the two groups, $A$ and $B$ separately. So the pivoted columns may benefit only one group. To address this, we adapt the pivoting strategy of \LRRQR to benefit the group that suffers the worse reconstruction error. Thus, we perform simultaneous \RRQR on $A$ and $B$ and at step-$i$ obtain the corresponding $Q^A(i) ,R^A(i)$ and $Q^B(i),R^B(i)$. We inspect the spectra of both $R^A(i)$ and $R^B(i)$ and choose the pivot based on the following strategy: we choose the right singular vector $v$ of either $R^A(i)$ or $R^B(i)$  corresponding to $\max \{\sigma_1(R^A_{22}(i)),\sigma_1(R^B_{22}(i))\}$, and select $\Pi$ such that $|(\Pi^Tv)|_1=\|v\|_{\infty}$. The algorithm pseudocode is available in Algorithm~\ref{alg:Fair_L_RRQR}. 

\begin{algorithm}[h]
\small
	\caption{\label{alg:Fair_L_RRQR}Fair L-{\sc RRQR}}
	
	Input: QR factorizations $A\Pi^A=Q^AR^A$, $B\Pi^B=Q^BR^B$, $k$\\
	Output: permutation $\Pi_k$ 
	\begin{algorithmic}[1]
		\FOR{i=1,\ldots,k}
		\STATE $R_{22}^A \gets R^A[i:,i:]$,  $R_{22}^B \gets R^B[i:,i:]$

  \STATE  $v \gets$ $\max \{\sigma_1(R^A_{22}(i)),\sigma_1(R^B_{22}(i))\}$
		\STATE Compute permutation $P$ such that $|(P^Tv)_1|=\|P^Tv\|_{\infty}$
		\STATE Compute QR fact. $R^A_{22}P=Q^A_1\tilde{R}^A_{22}$ and $R^B_{22}P=Q^B_1\tilde{R}^B_{22}$
		\STATE $\Pi \gets \Pi\begin{pmatrix}
		I & 0 \\ 0 & P
		\end{pmatrix}$
		\STATE $R^A \gets \begin{pmatrix}
		R^A_{11} & {Q_1^A}^TR^A_{12} \\ 0 & \tilde{R}^A_{22}
		\end{pmatrix}$ and $R^B \gets \begin{pmatrix}
		R^B_{11} & {Q_1^B}^TR^B_{12} \\ 0 & \tilde{R}^B_{22}
		\end{pmatrix}$
		\ENDFOR
	\end{algorithmic}
	{\bf return} $S$
\end{algorithm}

\section{Experiments}
\label{sec:experiments}
This section describes the experimental setup, datasets used, and presents
the experimental evaluation results.

{\bf Experimental setup.}
Our implementation is written in \texttt{python}. We use \texttt{numpy}, \texttt{scipy} and \texttt{scikit-learn} for preprocessing, linear algebra operations as well as parallelization. Experiments are executed on a compute node with 32 cores and 256GB of RAM. Our implementations are available as open source~\citep{sourcecode}.

{\bf Datasets.} We use juvenile recidivism data (\recidivism) from Cata\-lunya~\citep{tolan2019machine} and medical expenditure survey data 2015 (\meps)~\citep{cohen2009medical}, as well as various datasets from the UCI-ML repository~\citep{dua2019uci}:``heart-cleveland" (\heart), ``adult" (\adult), ``german-credit" (\german), ``credit-card" (\credit), ``student performance" (\student), ``compas-recidivism" (\compas), ``communities" (\communities). Data is processed by removing protected attributes, converting categorical variables to one-hot encoding and normalizing each column to unit $L^2$-norm. Group membership is based on Sex, except for \communities where group membership is majority white or a non-white community. Dataset statistics are reported in Table~\ref{table:dataset-stats}.

\begin{table}
\caption{Dataset statistics. $m_A$ and $m_B$ are the number of instances in groups $A$ and $B$, respectively. $n$ is the number of columns. $\gamma$($A$), $\gamma$($B$) is rank of $A$, $B$.}
\label{table:dataset-stats}
\centering
\small
\begin{tabular}{l r r r r r r}
\toprule
Dataset & $n$ & $m_A$ & $m_B$ & $\gamma$($A$) & $\gamma$($B)$\\
\midrule
      \heart &     14 &     201 &      96 &     13 &     13 \\
     \german &     63 &     690 &     310 &     49 &     47 \\
     \credit &     25 & 18\,112 & 11\,888 &     24 &     24 \\
    \student &     58 &     383 &     266 &     42 &     42 \\
      \adult &    109 & 21\,790 & 10\,771 &     98 &     98 \\
     \compas &    189 &  9\,336 &  2\,421 &    167 &     73 \\
\communities &    104 &  1\,685 &     309 &    101 &    101 \\
 \recidivism &    227 &  1\,923 &     310 &    175 &    113 \\
      \meps  & 1\,247 & 18\,414 & 17\,013 & 1\,217 & 1\,200 \\
\bottomrule
\end{tabular}
\end{table}

{\bf Experimental evaluation.}
Our experiments are in threefold. First, we assess the efficacy of the proposed algorithms in addressing the problem of \faircssminmax. This evaluation involves comparing the performance of the
proposed algorithms, considering various experimental setups.
Second, we evaluate the effectiveness of the \faircssminmax objective in selecting column subsets that result in fair reconstruction errors. Specifically, we compare the reconstruction errors of each group in the optimal solutions obtained using the vanilla \css objective versus the \faircssminmax objective.
Last, we investigate the \emph{price of fairness}. This entails verifying potential trade-offs or costs associated to attain fairness according to \faircssminmax.

\subsection{Algorithms Evaluation} 
{\bf Algorithms.}
We refer Algorithm~\ref{alg:fairnessScores} as \sampler and fair \LRRQR as \lowqr. We also consider the fair version of a variant of \LRRQR, called \HRRQR \cite{chan1987rank}, (\highqr). For details on this algorithm see Supplementary~\ref{sup:sec:fair_lrrqr}.
The complexity of our algorithms is dominated by SVD. At each step, \lowqr and \highqr require $\bigO(k)$ and $\bigO(n-k)$ time, respectively. On the other hand, \sampler computes SVD once for each group, and requires $\bigO(n\log n)$ for sorting tuples.

We complement our algorithms with a \greedy algorithm: at each step it picks the column with the highest direct gain according to \minmaxloss. The complexity is dominated by matrix multiplication $\bigO(n^\Omega)$. Finally, in \random, we randomly sample $k$-column set for $100$ repetitions and choose a set with best score. 

{\bf Picking $c\geq k$ columns.}
Recall that \sampler chooses columns based on threshold $\theta$, and it can choose $c\geq k$ columns. In the first experiment, we evaluate the algorithms performance with respect to a specific low-rank subspace of $A$ and $B$ over different values of $c$. Thus, we evaluate the performance according to \minmaxloss while keeping $\|A-A_k\|_F$ and $\|B-B_k\|_F$ fixed. We perform this for six largest datasets for $k=20$ (for \meps, $k=50$). Figure \ref{fig:obj_vs_c} shows the results.
\begin{figure*}[h]
	\centering
	\includegraphics[width=1.0\textwidth]{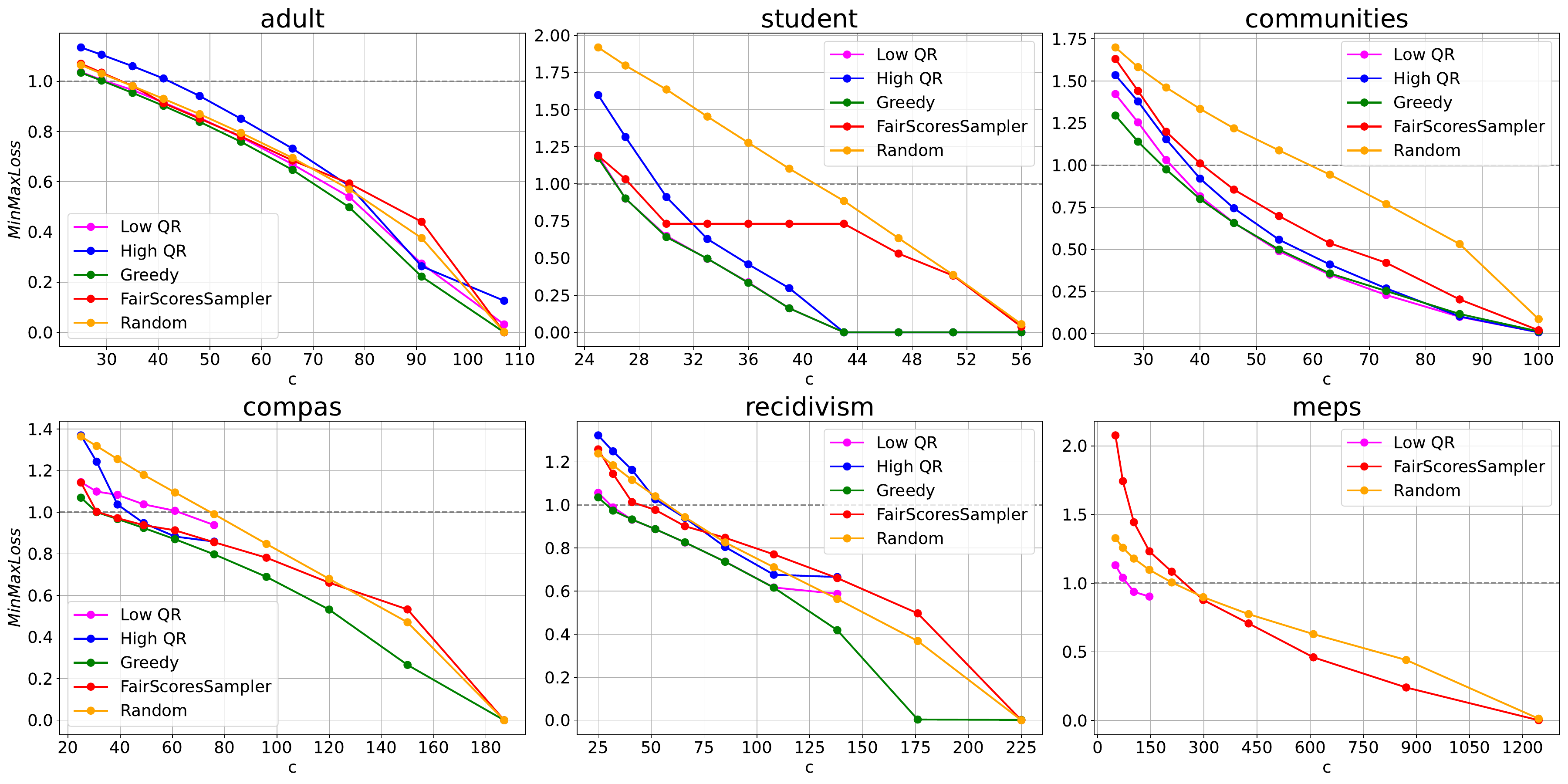}
    \caption{\minmaxloss for different values of $c$ and fixed target rank $k$}
    \label{fig:obj_vs_c}
\end{figure*}
Note that \lowqr and \highqr can only sample at most $\min(rank(A),rank(B))$ columns. \lowqr shows relatively good performance compared to \greedy, which has the best performance. In \meps, \greedy and \highqr, due to their higher complexity did not terminate within 24 hours. Note that in \meps, \sampler performs worse than \random for $c<300$, which implies the rank $50$ leverage scores do not decay quickly for $c<300$, and thus are not very informative.  In Supplementary~\ref{sup:sec:decay}, we plot the rank-$k$ leverage scores for datasets for different values of $k$. 

{\bf Two-stage sampling.}
Recall that, \faircssminmax asks for exactly $k$ columns. To combine the efficiency of \sampler and effectiveness of algorithms such as \greedy, we introduce and evaluate a two-stage sampling approach that returns exactly $k$-columns. Similar ideas have been explored in \css literature~\citep{boutsidis2008unsupervised}. In the first stage, we run \sampler that takes as input a threshold $\theta$ and the rank-$k$ leverage scores of $A$ and $B$, and returns $c\geq k$ columns. In the second stage, we run \lowqr, \highqr or \greedy on the subset of columns returned in the first stage, to obtain a column subset of size $k$. We refer to these methods as \slowqr, \shighqr and \sgreedy. 
\begin{table*}
	\caption{Performance comparison of algorithms.}
	\label{table:performance}
	\centering
\begin{tabular}{c@{\hspace{1em}}c@{\hspace{1em}}c@{\hspace{1em}}@{\hspace{1em}}c@{\hspace{1em}}c@{\hspace{1em}}c@{\hspace{1em}}c@{\hspace{1em}}c@{\hspace{1em}}c}
	Dataset & $c$ & $k$ & \multicolumn{6}{c}{MinMax Loss} \\ 
	& & & \slowqr
	& \shighqr
	& \sgreedy
	& \lowqr
	& \greedy
	& \random
	\\ 
	\toprule
	\multirow{3}{*}{\communities} & 89 & 10 & 1.27323 & 1.50763 & 1.17121 & 1.25939 & \textbf{1.16976} & 1.47485 \\ 
	& 94 & 23 & 1.32462 & 1.54204 & 1.2717 & 1.41701 & \textbf{1.23969} & 1.65518 \\ 
	& 98 & 51 & 1.40469 & 1.57761 & 1.42658 & 1.40934 & \textbf{1.39823} & 2.95985 \\ 
	\midrule
	\multirow{3}{*}{\compas} & 118 & 10 & 1.04747 & 1.2738 & 1.03356 & 1.05041 & \textbf{1.03057} & 1.19477 \\ 
	& 165 & 19 & 1.08617 & 1.34321 & 1.05688 & 1.08617 & \textbf{1.0537} & 1.27599 \\ 
	& 177 & 37 & 1.37174 & 1.41811 & 1.14835 & 1.47138 & \textbf{1.1291} & 1.67851 \\ 
	\midrule
	\multirow{3}{*}{\adult} & 70 & 10 & 1.02345 & 1.09485 & 1.02111 & 1.02345 & \textbf{1.01768} & 1.05641 \\ 
	& 96 & 22 & \textbf{1.03347} & 1.12764 & 1.0374 & \textbf{1.03347} & - & 1.0589 \\ 
	& 103 & 49 & \textbf{1.07796} & 1.19301 & 1.40252 & 1.08317 & - & 1.0994 \\ 
	\midrule
	\multirow{3}{*}{\german} & 53 & 10 & 1.08088 & 1.30176 & 1.08488 & 1.07711 & \textbf{1.07349} & 1.14205 \\ 
	& 54 & 15 & 1.1439 & 1.34599 & 1.11798 & 1.11871 & \textbf{1.11088} & 1.1966 \\ 
	& 54 & 24 & 1.20605 & 1.38489 & 1.192 & 1.20246 & \textbf{1.18624} & 1.36138 \\ 
	\midrule
	\multirow{3}{*}{\recidivism} & 134 & 10 & 1.02485 & 1.17864 & 1.02313 & 1.02236 & \textbf{1.01483} & 1.12757 \\ 
	& 174 & 24 & 1.05569 & 1.29805 & 1.04054 & 1.05567 & \textbf{1.03202} & 1.22332 \\ 
	& 212 & 57 & 1.31688 & 1.52031 & 1.16871 & 1.27495 & \textbf{1.13311} & 1.59933 \\ 
	\midrule
	\multirow{3}{*}{\student} & 45 & 10 & 1.10833 & 1.42094 & \textbf{1.10559} & 1.11333 & 1.10597 & 1.17856 \\ 
	& 46 & 14 & 1.14467 & 1.39265 & 1.14375 & 1.15592 & \textbf{1.14361} & 1.26605 \\ 
	& 47 & 21 & 1.18932 & 1.5756 & \textbf{1.17832} & 1.209 & 1.18771 & 1.56265 \\ 
	\midrule
	\multirow{3}{*}{\meps} & 428 & 10 & 1.14642 & 1.83209 & - & \textbf{1.05601} & - & 1.17759 \\ 
	& 382 & 32 & 1.20093 & 1.85843 & 1.777 & \textbf{1.11665} & - & 1.26749 \\ 
	& 338 & 100 & \textbf{1.33916} & 1.70488 & 2.48787 & - & - & 1.47403 \\ 
	\midrule
\end{tabular}
\end{table*}

Table \ref{table:performance} reports the results for the seven largest datasets for various values of $k$. We set $\theta = k-\frac{1}{2}$ in all cases, except \meps, where $\theta = \frac{3k}{4}$ to reduce the number of sampled columns. Thus in all datasets (except \meps) the resulting number of columns in the first stage, is a $\sqrt{2}$-approximation to the optimal number of columns $c$. 
Column ``$c$" in Table \ref{table:performance} indicates the number of columns returned in the sampling stage. Note that in some cases the number of columns is required to satisfy $\theta$ is significantly large. For each experiment, the best performing algorithm is highlighted in {\bf bold}.

We observe that \greedy performs better in most cases, but does not terminate always within twenty-four hours. \sgreedy is faster than \greedy, due to the sampling stage, and the objective values are close. On the other hand, \shighqr does not perform well in practice, though in theory we expect it to perform better for $k$ closer to $n$. Finally, we demonstrate that (see \meps) small $k$ does not mean fewer columns are sampled in the first stage, because the lower-rank leverage scores decay faster; thus more columns are required to satisfy threshold $\theta$. A visualization of the decay of leverage scores is reported in the Supplementary~\ref{sup:sec:decay}.

\begin{figure*}
\centering
\setlength{\tabcolsep}{0.01cm}
\renewcommand{\arraystretch}{0.01}
\begin{tabular}{ccc}
\includegraphics[width=0.33\textwidth]{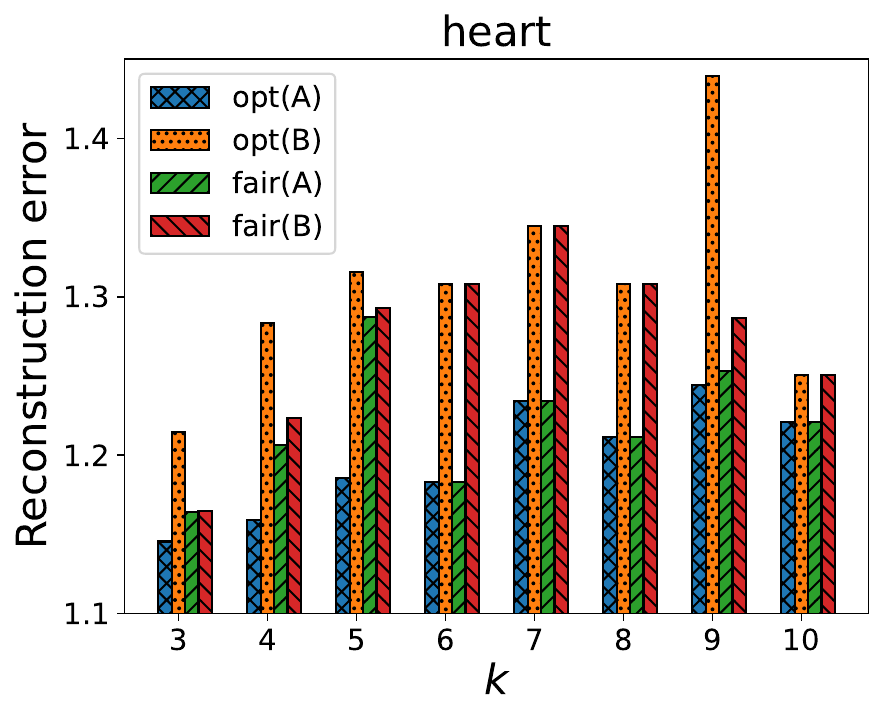} &
\includegraphics[width=0.33\textwidth]{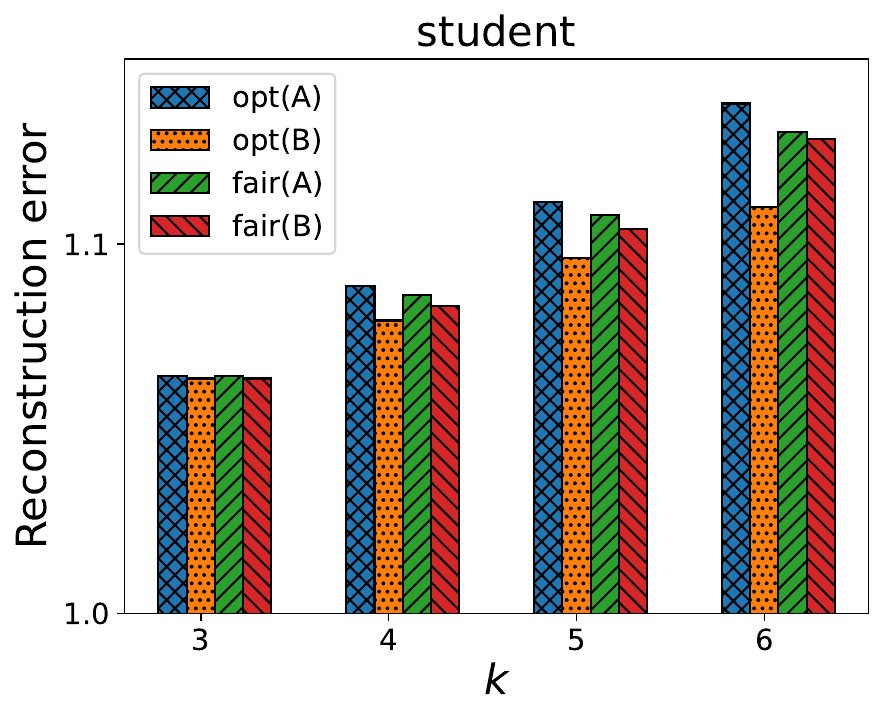} &
\includegraphics[width=0.33\textwidth]{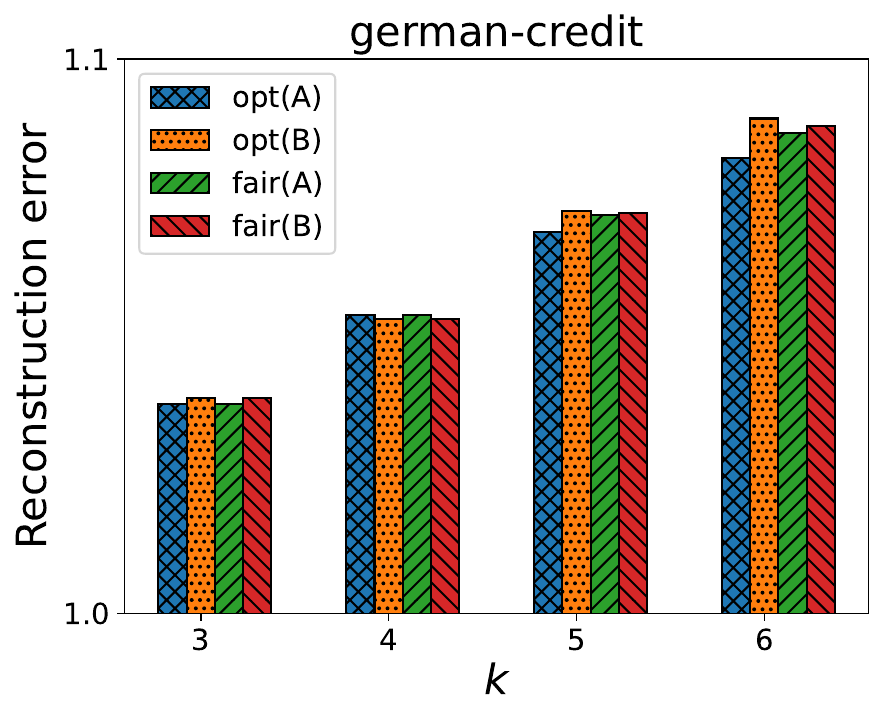}
\end{tabular}
\caption{Comparison of reconstruction error of \css and \faircss for groups $A$ and $B$.}
\label{fig:priceoffairness-1}
\end{figure*}

\begin{figure*}
\centering
\centering
\setlength{\tabcolsep}{0.01cm}
\renewcommand{\arraystretch}{0.01}
\begin{tabular}{ccc}
\includegraphics[width=0.33\textwidth]{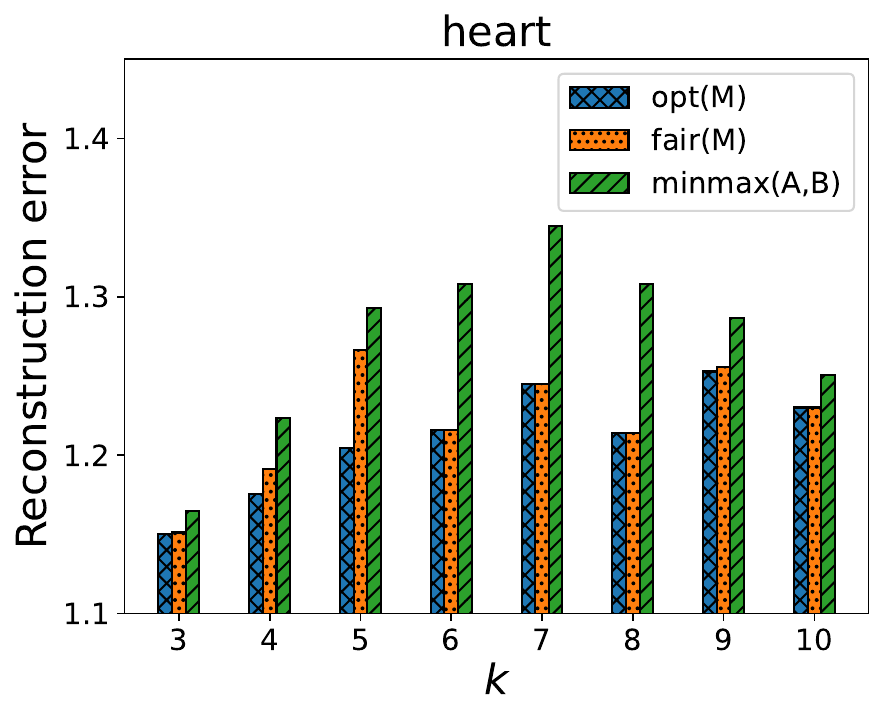} &
\includegraphics[width=0.33\textwidth]{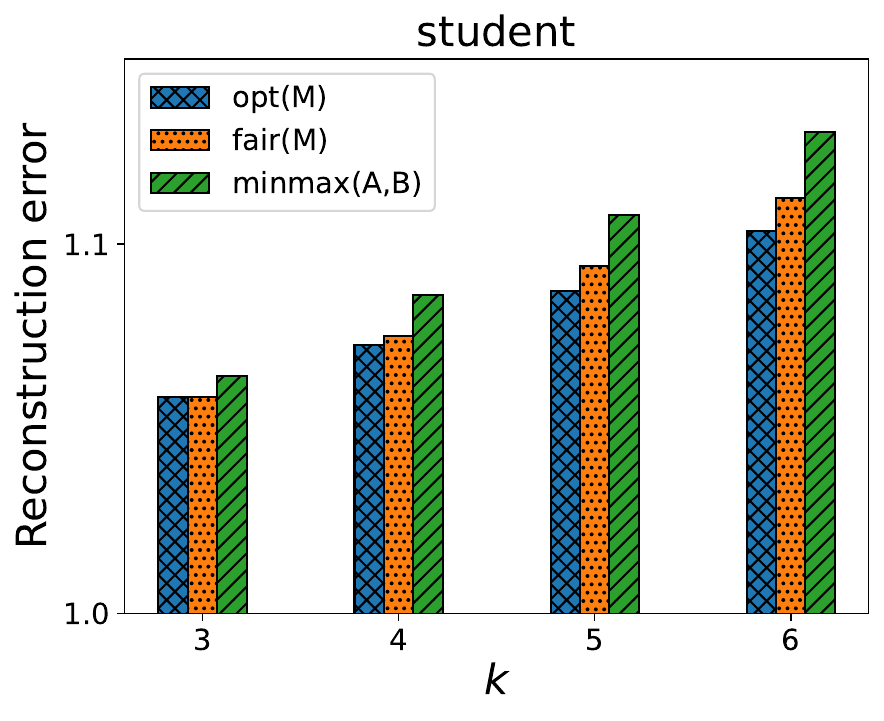} &
\includegraphics[width=0.33\textwidth]{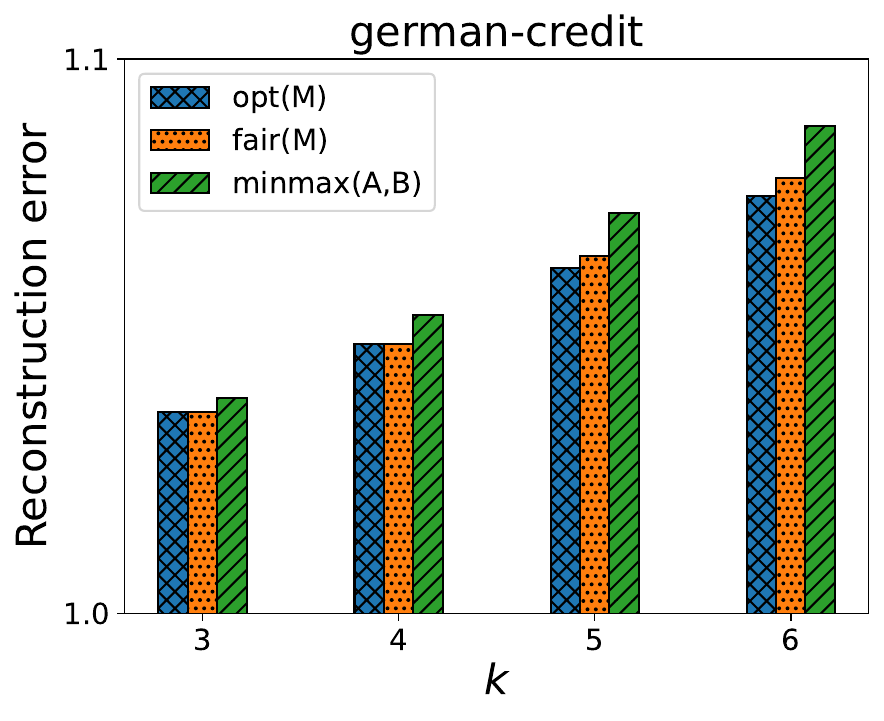}
\end{tabular}
\caption{Price of fairness.}
\label{fig:priceoffairness-2}
\end{figure*}

\subsection{Evaluation of the fair \css objective}
We examine the imbalance in reconstruction errors of the ``vanilla" \css objective of Problem~\ref{prob:CSS} and \faircssminmax objective for groups $A$ and $B$. We compute the optimal solution for each objective for various $k$ through exhaustive enumeration. $\opt(M), \opt(A)$ and $\opt(B)$ denote the reconstruction error of vanilla \css for matrices $M, A$ and $B$, respectively. $\minmax(M) = \minmax(A,B)$ denote the reconstruction error of matrix $M$. Further, $\fair(A)$ and $\fair(B)$ denote the reconstruction errors of \faircssminmax for groups $A$ and $B$, respectively. Lastly, $\fair(M)$ denote the reconstruction error of $M$ corresponding to the optimal solution of \faircssminmax. We note that brute-force enumeration is expensive, even after parallelization and extensive optimization, so we only report results for \heart, \student and \german datasets.

{\bf Fairness of solutions.}
In Figure~\ref{fig:priceoffairness-1} we compare $\opt(A)$, $\opt(B)$, that is, the optimal solution of vanilla \css, and the optimal solution of \faircssminmax $\fair(A), \fair(B)$. In most instances, we observe that the reconstruction errors of groups are disproportionate in vanilla
\css. However, the degree of imbalance is not uniform across datasets. 
One source of imbalance could be vastly different group sizes. However, as observed in \student, $|A| > |B|$, but $\opt(A) < \opt(B)$.
This further supports the need for sophisticated approaches to fairness in \css, beyond mere normalization.

{\bf Price of fairness.} 
Next, we verify how the fairness objective influences the quality of the solution in terms of reconstruction error. In Figure~\ref{fig:priceoffairness-2}, we report the optimal solution of vanilla \css and \faircssminmax to assess the trade-off between fairness and reconstruction error. This analysis quantifies the extent to which we sacrifice reconstruction error to achieve our fairness objective. In most cases, we observe no significant difference in reconstruction error between vanilla \css and \faircssminmax, that is, $\opt(M)$ and $\fair(M)$, even though the value of $\minmax(M)$ is significantly higher than $\opt(M)$. Note that, the observations cannot be generalized across all datasets, and we cannot conclusively claim that there is no trade-off between fairness and reconstruction error. For instance, in the case of the \heart dataset, we observe a significant difference between values of $\opt(M)$ and $\fair(M)$ for $k=5$.

\section{Conclusion} \label{sec:conclusion}
We introduced a novel \css variant for a fair setting when the matrix rows are partitioned into two groups. Our goal is to minimize the reconstruction error for both groups via a min-max objective. We utilized leverage scores to present an approximation algorithm when the column count is relaxed, and presented rank revealing \QR-factorisation-based algorithms when the column count is fixed. Extensive experiments on real-world data validated the effectiveness of our approach in improving fairness.
As future work, a natural direction is to extend our results to more than two groups. Also, improving the approximation ratio of \minfairnessscores requires further investigation.

\section*{Acknowledgements}
Suhas Thejaswi acknowledges support from the European Research Council (ERC) under the European Union'{}s Horizon $2020$ research and innovation program (grant agreement No. $945719$) and the European Unions'{}s SoBigData++ Transnational Access Scholarship.
Antonis Matakos acknowledges support from the Academy of Finland through the grant "Model Management Systems: Machine learning meets Database Systems"- MLDB (32511).

\bibliographystyle{ACM-Reference-Format}
\bibliography{citations}

\appendix
\section{PSEUDOCODE of Fair H-RRQR} \label{sup:sec:fair_lrrqr}
High rank revealing \QR-factorisation (\HRRQR) is similar to low rank revealing \QR-factorisation (\LRRQR), for details of \LRRQR see Section~\ref{sec:fair_qr} of the main paper. In \HRRQR, we begin with $R_{11}=R, R_{22}=0$ and recursively build $R_{22}$ by moving columns to the back. The fair variant of \HRRQR, at step-$i$, 
 chooses the right singular vector $v$ corresponding to $\min \{\sigma_i(R^A_{11}(i)),\sigma_i(R^B_{11}(i))\}$, where $\sigma_i$ is the bottom singular value. Then, we construct a permutation $\Pi_{i+1}$ such that $|(\Pi_{i+1}^Tv)|_i=\|v\|_{\infty}$. 

\begin{algorithm}
\small
	\caption{\label{alg:Fair_H_RRQR}Fair H-{\sc RRQR}}
	
	Input: QR factorizations $A\Pi^A=Q^AR^A$, $B\Pi^B=Q^BR^B$, $k$\\
	Output: permutation $\Pi_k$ 
	\begin{algorithmic}[1]
		\FOR{i=n,\ldots,n-k+1}
		\STATE $R_{11}^A \gets R^A[:i,:i]$,  $R_{11}^B \gets R^B[:i,:i]$
		\STATE  $v \gets$ $\min \{\sigma_i(R^A_{11}(i)),\sigma_i(R^B_{11}(i))\}$ 
		\STATE Compute permutation $P$ such that $|(P^Tv)_i|=\|P^Tv\|_{\infty}$
		\STATE Compute QR fact. $R^A_{11}P=Q^A_1\tilde{R}^A_{11}$ and $R^B_{11}P=Q^B_1\tilde{R}^B_{11}$
		\STATE $\Pi \gets \Pi\begin{pmatrix}
		P & 0 \\ 0 & I
		\end{pmatrix}$
		\STATE $R^A \gets \begin{pmatrix}
		\tilde{R}^A_{11} & {Q_1^A}^TR^A_{12} \\ 0 & R^A_{22}
		\end{pmatrix}$ and $R^B \gets \begin{pmatrix}
		\tilde{R}^B_{11} & {Q_1^B}^TR^B_{12} \\ 0 & R^B_{22}
		\end{pmatrix}$
		\ENDFOR
	\end{algorithmic}
	{\bf return} $S$
\end{algorithm}

\section{Decay of Leverage Scores}
\label{sup:sec:decay}

We plot the leverage scores of groups $A$ and $B$ for the experiments in Table \ref{table:performance}. The leverage scores are sorted separately for the two groups and plotted in decreasing order of their value.

\begin{figure*}[h]
	\centering
	\includegraphics[width=0.8\textwidth]{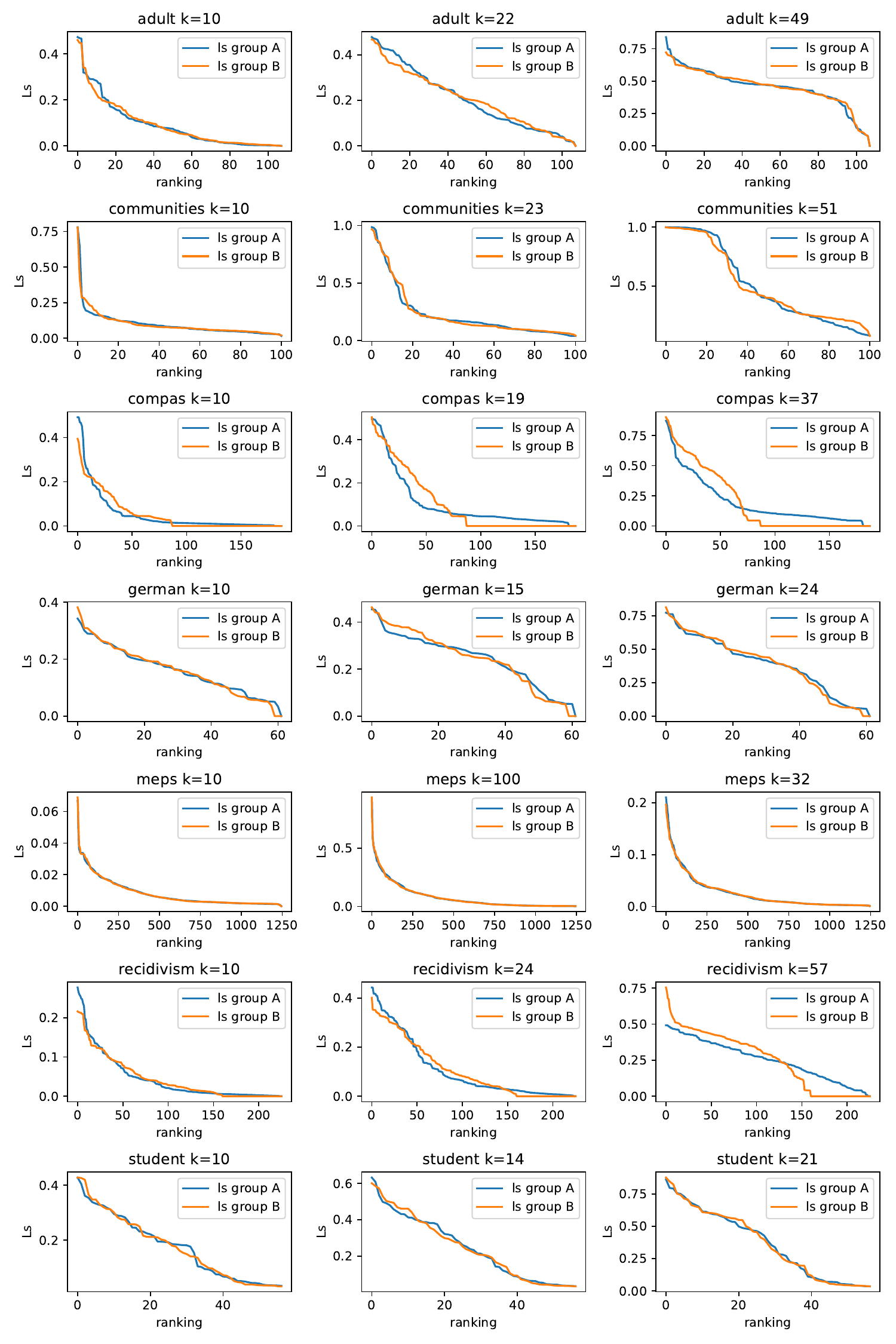}
	\caption{Leverage scores of $A$ and $B$ for Table \ref{table:performance}}
	\label{fig:ls}
\end{figure*}

\end{document}